\pdfoutput=1
\documentclass{article}
\usepackage{iclr2027_conference,times}

\usepackage[
  breaklinks=true,
  bookmarks=false,
  colorlinks=true,
  linkcolor=teal,
  urlcolor=teal,
  citecolor=magenta,
]{hyperref}

\usepackage{amsmath}
\usepackage{amsthm}
\usepackage{amssymb}
\usepackage{xcolor}
\usepackage{booktabs}
\AddToHook{env/table/begin}{%
  \setlength{\abovecaptionskip}{0pt}%
  \setlength{\belowcaptionskip}{\medskipamount}%
}

\usepackage{graphicx}
\usepackage{wrapfig}
\usepackage[capitalize]{cleveref}
\crefname{section}{Sec.}{Sec.}
\crefname{subsection}{Sec.}{Sec.}
\crefname{table}{Tab.}{Tab.}
\crefname{figure}{Fig.}{Fig.}

\newtheorem{theorem}{Theorem}
\newtheorem{proposition}{Proposition}

\usepackage{tikz}
\usetikzlibrary{cd}

\title{A Unifying Framework of Concept-based Explainable AI with Completeness Guarantees}

\author{%
\And
Vojt\v{e}ch K\r{u}r
\And
Adam Kuku\v{c}ka
\And
Tom\'{a}\v{s} Br\'{a}zdil
\And
V\'{\i}t Musil
\And
\end{tabular}\parfillskip=0pt\endgraf
\centering
\begin{tabular}[t]{@{}c@{}}
Masaryk University, Faculty of Informatics, Brno, Czechia
}

\DeclareMathOperator{\rmse}{RMSE}
\DeclareMathOperator{\recerror}{RE}
\DeclareMathOperator{\modelerror}{MCE}
\DeclareMathOperator{\fiderror}{FE}
\DeclareMathOperator{\atterror}{ATE}
\DeclareMathOperator{\adderror}{ADD}
\DeclareMathOperator{\id}{id}

\iclrfinalcopy
\begin{document}

\maketitle
\lhead{Preprint}

\begin{abstract}
Concept-based explanations describe neural network predictions through human-understandable properties of inputs called \emph{concepts}.
The field encompasses approaches that differ in how they define and represent concepts and connect them to model predictions.
We introduce a \emph{theoretical framework} that describes these approaches in a common mathematical language and supports a shared analysis of their properties.
For \emph{concept discovery}, which identifies concepts automatically within a latent space of a trained model, we employ a \emph{concept autoencoder} view.
An \emph{encoder} extracts concept representations from the model's latent space, and a \emph{decoder} uses them to reconstruct the original latent representation.
The autoencoder's \emph{reconstruction error} measures how accurately its decoder recovers the original latent representation.
We revisit \emph{model completeness}: how well the concepts can reproduce the model's outputs.
We show that model incompleteness of the concepts can be bounded by the autoencoder's reconstruction error.
The autoencoder view also provides a common way to define individual concept \emph{attributions}, which measure each concept's contribution to a prediction.
We establish when these attributions sum to the model's prediction, and bound the discrepancy otherwise, thus providing \emph{attribution completeness} guarantees.
\end{abstract}

\section{Introduction}
\label{sec:introduction}

Deep learning has driven substantial advances across many tasks through its ability to learn rich representations directly from data~\citep{lecun2015deep}.
However, these representations are not inherently aligned with human-understandable concepts, making the resulting predictions difficult to interpret~\citep{samek2017explainable}.
Explainable artificial intelligence (XAI) seeks to make model behavior and predictions understandable to humans.
Concept-based XAI (C-XAI) approaches this goal through properties of inputs called \emph{concepts}, such as textures, shapes, or object parts~\citep{poeta2023concept}.
For example, the presence of wheels may help explain why an image is classified as a vehicle.

C-XAI encompasses methods that differ in what a concept is, how it is represented, and how it relates to predictions~\citep{poeta2023concept}. 
These different uses and representations of concepts have motivated mathematical frameworks for relating and comparing concept-based methods~\citep{fel2023holistic,li2024readability,poche2025consim}.
We focus on two main problems in bridging C-XAI together: a) a common definition of a \emph{concept}, and b) a suitably general definition of \emph{concept extraction}.
In the following, we present our framework of C-XAI.

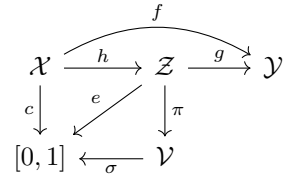
\begin{wrapfigure}[9]{r}{0.27\textwidth}
    \centering
    \vspace{-2ex}
    \hspace*{-5pt}%
    \begin{tikzcd}
        \mathcal{X} \arrow[r, "h"] \arrow[rr, bend left=30, "f"] \arrow[d, "c"'] & \mathcal{Z} \arrow[r, "g"] \arrow[d, "\pi"] \arrow[dl, "e"'] & \mathcal{Y} \\
        {[0,1]} & \mathcal{V} \arrow[l, "\sigma"] &
    \end{tikzcd}
    \let\normalsize\small
    \caption{\centering Concept extraction.}
    \label{fig:concept-diagram}
\end{wrapfigure}
We define a concept to be a function $c\colon\mathcal{X}\to[0,1]$, where $\mathcal{X}$ is the input space.
Binary values indicate absence or presence of the concept in the input, while intermediate values allow graded descriptions.
We write a model as $\smash{f=g\circ h\colon\mathcal{X}\xrightarrow{h}\mathcal{Z}\xrightarrow{g}\mathcal{Y}}$, where $h$ maps inputs to latent representations and $g$ maps these representations to predictions.
We define a \emph{concept extractor} as a map $e\colon\mathcal{Z}\to[0,1]$, so that $c=e\circ h$ defines a concept on inputs.
However, existing methods can represent a single concept through a vector or a collection of spatial responses, containing information beyond its scalar presence~\citep{vielhaben2023multidimensional,fel_craft_2023}.
We therefore separate the extracted representation from its activation by writing $e=\sigma\circ\pi$: a \emph{probe} $\pi\colon\mathcal{Z}\to\mathcal{V}$ produces the concept representation, while an \emph{activation function} $\sigma\colon\mathcal{V}\to[0,1]$ measures its presence.
This allows $\mathcal{V}$ to be multidimensional while ensuring that $c=\sigma\circ\pi\circ h$ remains a concept under our input-level definition (\cref{fig:concept-diagram}).

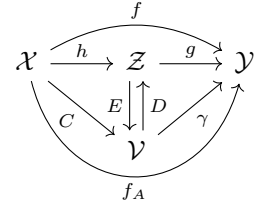
\begin{wrapfigure}[10]{r}{0.25\textwidth}
    \centering
    \vspace{-2.5ex}
    \hspace*{-5pt}%
    \begin{tikzcd}
        \mathcal{X} \arrow[r, "h"] \arrow[rr, bend left=30, "f"] \arrow[dr, "C"'] \arrow[rr, bend right=75, looseness=1.6, "f_A"'] & \mathcal{Z} \arrow[r, "g"] \arrow[d, shift right, "E"'] & \mathcal{Y} \\
        & \mathcal{V} \arrow[u, shift right, "D"'] \arrow[ur, "\gamma"'] &
    \end{tikzcd}
    \let\normalsize\small
    \caption{\centering Concept autoenc.}
    \label{fig:cae-diagram}
\end{wrapfigure} 
Unsupervised concept discovery methods learn probes $\pi_1,\ldots,\pi_k$ that together form an \emph{encoder} $E(z)=(\pi_1(z),\ldots,\pi_k(z))$.
The encoder maps the model's latent space $\mathcal{Z}$ to the joint concept representation space $\mathcal{V}=\mathcal{V}_1\times\cdots\times\mathcal{V}_k= \prod_i\mathcal{V}_i$.
Following earlier frameworks~\citep{fel2023holistic,poche2025consim}, we pair it with a \emph{decoder} $D\colon\mathcal{V}\to\mathcal{Z}$ that reconstructs the latent representation, forming a \emph{concept autoencoder} $A=(E,D)$.
The autoencoder also defines the \emph{induced concept autoencoder model} (ICAM) $f_A=g\circ D\circ E\circ h$.
Denoting $\gamma = g\circ D$ and $C = E \circ h$, we have $f_A = \gamma \circ C$ (\cref{fig:cae-diagram}).

We measure the \emph{reconstruction error} $\recerror(A)$ by the root mean squared error (RMSE) between $h$ and $D\circ E \circ h$ over a distribution of inputs.
Following ICE's comparison of reconstructed and original predictions~\citep{zhang2021invertible}, we measure \emph{fidelity error} $\fiderror(A)$ by the RMSE between $f$ and $f_A$.
We revisit \emph{model completeness}: whether the concept representations $E(h(x))$ retain enough information to recover the model's outputs $f(x)$~\citep{yeh2020completeness}.
We define the \emph{model completeness error} $\modelerror(A)$ as the infimum of the RMSE between $f$ and $\eta \circ C$, over all eligible $\eta \colon \mathcal{V} \to \mathcal{Y}$, i.e.,
\begin{equation}
    \modelerror(A)=\inf_{\eta}\rmse(f, \eta \circ C).
\end{equation}
Choosing $\eta = \gamma$ gives $f_A$, so fidelity error bounds MCE.
If $g$ is $L_g$-Lipschitz continuous, a change in the latent representation produces at most a proportional change in the output.
Together this makes the following chain of inequalities:
\begin{equation}
\label{eq:intro-bound}
    \modelerror(A)\leq \fiderror(A) \leq L_g\recerror(A).
\end{equation}
 
For a scalar output $\mathcal{Y} = \mathbb{R}$ (e.g., a class score), an \emph{attribution function} $\varphi=(\varphi_1,\ldots,\varphi_k)\colon\mathcal{X}\to\mathbb{R}^k$ assigns each concept a score intended to describe its contribution to the prediction~\citep{fel2023holistic,poche2025consim}.
We also revisit \emph{attribution completeness}: whether concept attributions together with a baseline sum to $f(x)$~\citep{sundararajan2017axiomatic, vielhaben2023multidimensional}.
We measure the discrepancy by the \emph{attribution error} $\atterror(A,\varphi)$, the RMSE between $f$ and $f_\varphi$ where $\smash{f_\varphi(x) = b + \sum_{i=1}^k\varphi_i(x)}$ and $b$ is a baseline value.
We consider three attribution rules: occlusion and gradient-times-input, as defined by \citet{fel2023holistic,poche2025consim}, and insertion.

\begin{figure}[tb]
    \centering
    \includegraphics[width=\linewidth]{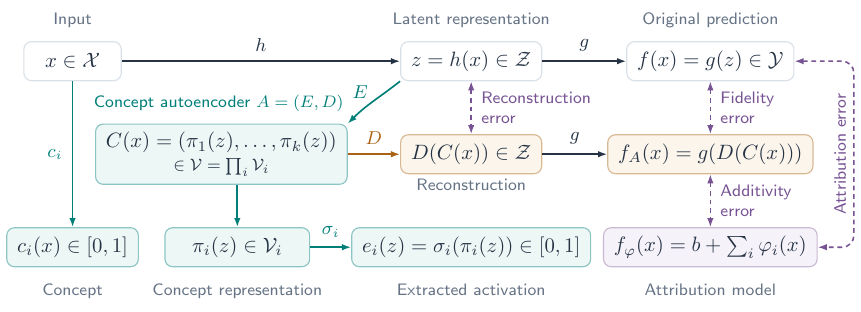}
    \caption{Overview of the framework for a scalar model output.
    The decomposition $f=g\circ h$ connects inputs, latent representations, and predictions.
    Concept probes $\pi_i$ and activation functions $\sigma_i$ form extractors $e_i=\sigma_i\circ\pi_i$ that estimate input-level concepts $c_i$.
    The concept autoencoder $A=(E,D)$ induces the reconstructed predictor $f_A$, while the attribution model $f_\varphi$ adds the concept contributions to a baseline $b$.
    Solid arrows denote function evaluations; dashed links compare the indicated functions by RMSE over the input distribution.
    Model completeness instead considers the best recovery of $f$ from the concept representations over prediction heads.}
    \label{fig:framework-overview}
\end{figure}

The three attribution rules are defined using the concept autoencoder $A$.
For example, occlusion measures the difference between $f_A(x)$ and the same value with the $i$-th concept set to zero.
We make the relationship of $\varphi$ to $A$ explicit using the triangle inequality
\begin{equation}
    \atterror(A, \varphi) = \rmse(f,f_{\varphi}) \leq \rmse(f, f_A) + \rmse(f_A, f_{\varphi}) = \fiderror(A) + \adderror(A, \varphi),
\end{equation}
where $\adderror(A,\varphi) = \rmse(f_A, f_\varphi)$ is the \emph{additivity error}.
In the case that $\gamma$ is affine, for example when $D$ is linear and $g$ is a final linear layer, we show that all three rules are equal and $\adderror(A,\varphi) = 0$.
This means that the attribution rules perfectly attribute the ICAM, and their error with respect to the original model can be bounded in the same way as MCE in \cref{eq:intro-bound}.
Finally, we also show a bound on the additivity error when $\gamma$ is not affine using a bound on the curvature of $\gamma$.
The key terms are summarized in \cref{fig:framework-overview}.

\paragraph{Contributions}
We present a unifying mathematical framework for concept-based explanation methods.
This abstraction enables general completeness guarantees for the concept representations and attributions, under explicit assumptions.
Our contribution is the shared formulation and analysis that support these guarantees, rather than a new explanation method.
Specifically, our contributions are threefold:
\begin{enumerate}
    \item \textbf{A common formulation of concept-based explanation.}
    We distinguish concepts on inputs, their potentially multidimensional representations, and scalar activations within a shared formulation of concept testing, discovery, and bottleneck models.
    \item \textbf{Reconstruction-based model completeness guarantees.}
    We show that fidelity error bounds model completeness error and is itself bounded by reconstruction error under a Lipschitz assumption on the prediction head.
    Exact reconstruction yields zero fidelity and model completeness errors.
    \item \textbf{Attribution completeness through fidelity and additivity.}
    We bound attribution error by fidelity error plus additivity error within the ICAM.
    For affine decoded prediction heads, insertion, occlusion, and gradient-times-input give identical contributions with zero additivity error.
    For nonaffine decoded heads with a Lipschitz gradient, we bound additivity error using curvature and the magnitude of the concept representations.
\end{enumerate}

\section{Proposed Framework}

In this section, we give the formal description of our framework.

\paragraph{Concept}
A \emph{concept} is a function $\mathcal{X} \to [0,1]$, where $\mathcal{X}$ is an \emph{input space}.
In the following, we assume that $\mathcal{X}$ is fixed.

\paragraph{Model}
A \emph{model} is a function $\mathcal{X} \to \mathcal{Y}$, where $\mathcal{Y}$ is the \emph{output space}.
A \emph{decomposition} of a model $f$ is $(\mathcal{Z}, h, g)$ where $\mathcal{Z}$ is a \emph{latent space}, $h \colon \mathcal{X} \to \mathcal{Z}$ is a \emph{latent extractor}, $g\colon \mathcal{Z} \to \mathcal{Y}$ is a \emph{prediction head}, and $f = g\circ h$.
In the following, we assume the model and its decomposition are fixed.

\paragraph{Concept Extractor}
A \emph{concept extractor} is a function $\mathcal{Z} \to [0,1]$.
A \emph{decomposition} of a concept extractor $e$ is $(\mathcal{V}, \pi, \sigma)$, where $\mathcal{V}$ is a \emph{representation space}, $\pi \colon \mathcal{Z} \to \mathcal{V}$ is a \emph{concept probe}, $\sigma \colon \mathcal{V} \to [0,1]$ is a \emph{concept activation function}, and $e = \sigma \circ \pi$.

\paragraph{Concept Autoencoder}
A $k$-\emph{concept autoencoder} is $(E, D)$, where $E = (\pi_1, \ldots, \pi_k)\colon\mathcal{Z} \to \mathcal{V}$ is a $k$-\emph{concept encoder}, $\mathcal{V} = \mathcal{V}_1 \times \cdots \times \mathcal{V}_k$ is a \emph{product representation space}, $\pi_i\colon \mathcal{Z} \to \mathcal{V}_i$ is the $i$-th \emph{concept probe}, and $D \colon \mathcal{V} \to \mathcal{Z}$ is a $k$-\emph{concept decoder}.
An \emph{induced concept autoencoder model} (ICAM) $f_A$ of a concept autoencoder $A=(E, D)$ is $g \circ D \circ E \circ h$.

\section{Overview of C-XAI}
\label{sec:background}

Concept-based explanations differ in how concepts enter the model, how they are obtained, and what the explanation describes~\citep{poeta2023concept}.
Concepts may participate in prediction (\emph{ante-hoc}) or be used to analyze a trained model (\emph{post-hoc}).
They may be specified through annotations or descriptions (\emph{supervised}), or discovered without concept labels (\emph{unsupervised}).
In this section, we show how these methods can be described using our framework; \cref{tab:instantiations} in \cref{app:instantiations} summarizes the instantiations.

\subsection{Explainable-by-Design Models}

Concept-based models use concepts directly in prediction.
We distinguish their architectures independently of how the concepts are obtained or the models are trained.

\begin{wrapfigure}[14]{r}{0.31\textwidth}
    \centering
    \vspace{-2.5ex}
    \hspace*{-5pt}%
    \begin{tikzcd}
        \mathcal{X} \arrow[r, "{(c_1,\ldots,c_k)}"] \arrow[dr, "c_i"'] &[1.6em] {[0,1]^k} \arrow[r, "g"] \arrow[d, "\pi_i", start anchor={[xshift=-2.5pt]south}, end anchor=north] &[-0.6em] \mathcal{Y} \\[-0.5em]
        & |[xshift=-2.5pt]| {[0,1]} &
    \end{tikzcd}
    \let\normalsize\small
    \caption{\centering Bottleneck model.}
    \label{fig:cbm-diagram}
    \bigskip
    \hspace*{-5pt}%
    \begin{tikzcd}
        \mathcal{X} \arrow[r, "h"] \arrow[d, "c_i"'] & \prod_j\mathcal{V}_j \arrow[r, "g"] \arrow[d, "\pi_i"] & \mathcal{Y} \\[-0.5em]
        {[0,1]} & \mathcal{V}_i \arrow[l, "\sigma_i"] &
    \end{tikzcd}
    \caption{\centering Embedding model.}
    \label{fig:cem-diagram}
\end{wrapfigure}
A \emph{concept bottleneck model} (CBM) is a model $f$ paired with a decomposition $(\mathcal{Z}, h, g)$ of $f$ where $\mathcal{Z} = [0,1]^k$ and $h =(c_1,\ldots,c_k)$.
Hence, each $c_i=\pi_i\circ h$ is a concept, where $\pi_i(z)=z_i$ is the $i$-th coordinate projection, and the prediction head $g\colon[0,1]^k\to\mathcal{Y}$ receives their activations (\cref{fig:cbm-diagram}).

A \emph{concept embedding model} (CEM) is a model $f$ paired with a decomposition $(\mathcal{Z},h,g)$ where $\mathcal{Z}=\mathcal{V}_1\times\cdots\times\mathcal{V}_k$ and $k$ concept extractors $e_i$ with decompositions $(\mathcal{V}_i,\pi_i,\sigma_i)$, where $\pi_i(z)=z_i$.
Each probe $\pi_i$ selects a concept representation $\pi_i(h(x)) \in \mathcal{V}_i$, and $\sigma_i\colon\mathcal{V}_i\to[0,1]$ maps it to a scalar activation.
Thus $c_i = \sigma_i \circ \pi_i \circ h$ is a concept (\cref{fig:cem-diagram}).
The prediction head $g$ receives the representations $h(x)$ themselves.
Every CBM is a CEM with $\mathcal{V}_i=[0,1]$ and identity activations.

The probability-based variants of \citet{koh2020concept} and the prototype-presence constructions of ProtoTree and PIP-Net give CBM decompositions~\citep{nauta2021neural,nauta2023pipnet}.
Concept logits, as used in the sequential and joint classification variants of \citet{koh2020concept}, illustrate scalar CEM representations.
This construction also accommodates the scalar projections of Post-hoc CBMs without a residual branch and Label-free CBMs~\citep{yuksekgonul2023posthoc,oikarinen2023labelfree}, and the similarity scores of ProtoPNet and LaBo~\citep{chen2019looks,yang2023language}.
The vector-valued embeddings of \citet{zarlenga2022concept} illustrate multidimensional CEM representations, which can retain predictive information beyond their scalar concept activations.

\paragraph{Example: Supervised CBMs and CEMs}
We express the classification models of \citet{koh2020concept} in our framework.
The user specifies $k$ binary concepts and supplies training data $\{(x_j,y_j,a_j)\}_{j=1}^n$, where $x_j\in\mathcal{X}$ is an input, $y_j\in\{1,\ldots,N\}$ is its class label, and $a_j\in\{0,1\}^k$ contains its concept annotations.
The entry $a_{j,i}$ indicates whether concept $i$ is present in input $x_j$.

In the probability-bottleneck variant, the method learns concept predictors $\hat c_i\colon\mathcal{X}\to[0,1]$ and a prediction head $g\colon[0,1]^k\to\mathbb{R}^N$ producing class logits, so $\mathcal{Y}=\mathbb{R}^N$.
Concept annotations encourage $\hat c_i(x_j)\approx a_{j,i}$, while class labels encourage $g(\hat c_1(x_j),\ldots,\hat c_k(x_j))$ to predict $y_j$.
The resulting model predicts entirely through these concept probabilities and therefore has the CBM decomposition $\bigl([0,1]^k,(\hat c_1,\ldots,\hat c_k),g\bigr)$.

The other variants instead learn a concept-logit predictor $h\colon\mathcal{X}\to\mathbb{R}^k$ and a prediction head $g\colon\mathbb{R}^k\to\mathbb{R}^N$ that receives the logits directly.
This gives the model decomposition $(\mathbb{R}^k,h,g)$.
The architecture uses the fixed sigmoid $s(t)=(1+\exp(-t))^{-1}$ to obtain concept probabilities from the logits.
In our notation, $\mathcal{V}_i=\mathbb{R}$, the probe $\pi_i(z)=z_i$ selects the representation of concept $i$, and $\sigma_i=s$ gives its activation.
Concept supervision encourages $e_i(h(x_j))=s(h(x_j)_i)\approx a_{j,i}$, while the prediction head receives $h(x_j)$ itself.
The decomposition together with these extractors is therefore a CEM in our terminology.

\subsection{Post-hoc Concept-Based Explanations}

Post-hoc methods explain a fixed model $f$ with decomposition $(\mathcal{Z},h,g)$ by identifying concepts in its latent representations and examining their relation to predictions.

\subsubsection{Concept Detection}

Concept detection asks whether a trained model's latent representations $h(x)$ retain information about a concept of interest.
For a user-specified concept $c\colon\mathcal{X}\to[0,1]$, we seek an extractor $e=\sigma\circ\pi$ such that $e(h(x))\approx c(x)$.
Annotated examples allow this correspondence to be learned and evaluated.

The most direct approach associates concepts with existing model units.
For $\mathcal{Z}=\prod_i\mathcal{V}_i$, projection probes $\pi_i(z)=z_i$ select neurons, spatial feature maps, or groups of units, whose concept responses are described by $\sigma_i$.
This gives a CEM interpretation of the existing model decomposition.
Network Dissection evaluates channel--concept associations against annotated spatial masks~\citep{bau2017network}.
CRP similarly treats channels as concept representations~\citep{achtibat2023attribution}.

More generally, learned probes $\pi_i\colon\mathcal{Z}\to\mathcal{V}_i$ can use the full latent representation.
We can compare learned probes to study relationships between concepts, as in Net2Vec~\citep{fong2018net2vec}, or evaluate class-score sensitivity along concept directions, as in TCAV~\citep{kim2018interpretability}.
CAR extends detection to nonlinear probes and studies concept--class associations and input features supporting concept detection~\citep{crabbe2022car}.

\paragraph{Example: TCAV}
The user supplies a trained classifier $f\colon\mathcal{X}\to\mathbb{R}^N$ and selects a layer giving the decomposition $(\mathbb{R}^d,h,g)$.
They specify a binary concept $c$ through labeled examples $\{(x_j,a_j)\}_{j=1}^n$, where $a_j=c(x_j)\in\{0,1\}$.
In our framework, TCAV~\citep{kim2018interpretability} uses an affine probe $\pi(z)=\langle w,z\rangle+\beta$ with representation space $\mathcal{V}=\mathbb{R}$.
Taking $\sigma(t)=\mathbf{1}\{t>0\}$ gives the concept extractor $e=\sigma\circ\pi$.
The method learns $w$ and $\beta$ so that $e \circ h$ is as close as possible to $c$ on the labeled examples.

\subsubsection{Unsupervised Concept Discovery}
\label{sec:background-unsupervised}

Unsupervised concept discovery constructs concept representations without annotations specifying their meanings.
We use our framework to describe methods that admit reconstruction, followed by a concrete example using ICE.
For a fixed model decomposition $(\mathcal{Z},h,g)$, these methods construct probes $E=(\pi_1,\ldots,\pi_k)$ and a decoder $D$, forming a concept autoencoder $A=(E,D)$~\citep{fel2023holistic,poche2025consim}.
Many methods fit these functions by minimizing reconstruction error together with regularization or constraints~\citep{fel2023holistic,bhusal2025face}.
Other constructions provide exact reconstruction algebraically~\citep{vielhaben2023multidimensional}.
The discovered representations are then interpreted through representative images, spatial responses, or concept-specific saliency maps, allowing users to identify and name the concepts~\citep{zhang2021invertible,fel_craft_2023}.

The decoder also identifies the predictor whose concept contributions are being explained.
Writing $C=E\circ h$ and $\gamma=g\circ D$, the reconstructed predictor is the ICAM $f_A=\gamma\circ C$.
When activation functions $\sigma_i$ are specified, this gives the CEM decomposition $(\mathcal{V},C,\gamma)$, with input-level concepts $c_i=\sigma_i\circ\pi_i\circ h$.
Standard feature-attribution methods can therefore treat each concept representation as a feature block of this predictor.
For example, insertion and occlusion modify blocks of $C(x)$ before evaluating $\gamma$, while gradient-times-input combines these blocks with derivatives of $\gamma$~\citep{fel2023holistic,poche2025consim}.
These attributions describe the reconstructed prediction $f_A(x)$, which may differ from the original prediction $f(x)$.
The next section uses this distinction to study how well the concept representations and their attributions recover the original model's outputs.

\paragraph{Example: ICE}
We express ICE's NMF construction in our framework~\citep{zhang2021invertible} (see also \cref{fig:ice-diagram} in \cref{app:instantiations}).
The user supplies a trained classifier with $N$ classes, selects a class $r\in\{1,\ldots,N\}$, provides inputs $X=(x_1,\ldots,x_n)\in\mathcal{X}^n$ from that class, and chooses the number of concepts $k$.
We explain the class logit $f\colon\mathcal{X}\to\mathbb{R}$ using the nonnegative feature maps preceding global average pooling and the final affine classifier.
This gives the decomposition $(\mathcal{Z},h,g)$ with $\mathcal{Z}=\mathbb{R}_{\geq0}^{S\times d}$, where $S$ is the number of spatial locations and $d$ is the number of feature channels.
The prediction head is $g(z)=b+\frac{1}{S}\sum_{s=1}^{S}\langle w,z_{s,:}\rangle$, where $w\in\mathbb{R}^d$ and $b\in\mathbb{R}$ are the trained classifier's parameters for class $r$.
The original model remains fixed.

ICE stacks the supplied feature maps into $Z\in\mathbb{R}_{\geq0}^{nS\times d}$ and fits a dictionary $U\in\mathbb{R}_{\geq0}^{k\times d}$ and coefficients $H\in\mathbb{R}_{\geq0}^{k\times nS}$ by minimizing $\|Z-H^\top U\|_F^2$.
Each row of $U$ is a learned concept direction.
Once $U$ is fixed, a new input receives one coefficient per concept and spatial location.
We therefore take $\mathcal{V}_i=\mathbb{R}_{\geq0}^{S}$ and identify $\mathcal{V}=\prod_i\mathcal{V}_i$ with $\mathbb{R}_{\geq0}^{k\times S}$ by placing concept representations in rows.
ICE's reconstruction operation is our decoder $D(V)=V^\top U$.
Its numerical procedure for fitting coefficients with $U$ fixed gives our encoder $E(z)$, seeking to minimize $\|z-D(V)\|_F^2$ over $V\in\mathcal{V}$.
The probes are $\pi_i(z)=E(z)_{i,:}$, so each concept representation is an entire spatial coefficient map.
ICE uses these maps to highlight image regions and their spatial means to select representative images for interpreting the concepts.

For an input $x$, ICE multiplies the spatial mean of $\pi_i(h(x))$ by the class-specific weight $(Uw)_i$ to obtain concept $i$'s contribution.
Its explanation displays these quantities as the concept's similarity score, weight, and contribution.
Our general results in \cref{sec:cg-attribution} explain why these contributions, together with the classifier bias, recover the reconstructed logit and how reconstruction quality controls their agreement with the original prediction.

\section{Completeness Guarantees}
\label{sec:completeness-guarantees}

We connect reconstruction quality to fidelity, model completeness, and attribution completeness.
The results apply to the fitted encoder and decoder, independently of their learning objective or optimization procedure.

\paragraph{Common Assumptions and Notation}
Fix a model $f$ with decomposition $(\mathcal{Z},h,g)$, a concept autoencoder $A=(E,D)$, and a probability distribution $p$ on $\mathcal{X}$.
Write $C=E\circ h$ and $\gamma = g \circ D$, thus $f_A = \gamma \circ C$.
Assume that $\mathcal{Y}=\mathbb{R}^N$, $\mathcal{Z}\subseteq\mathbb{R}^d$, that all relevant functions are measurable, and that the errors being bounded are finite.
For functions $f_1,f_2\colon\mathcal{X}\to\mathbb{R}^m$, define
\begin{equation}
\label{eq:cg-rmse}
    \rmse(f_1,f_2)
    =\sqrt{\mathbb{E}_{x\sim p}
      \bigl[\|f_1(x)-f_2(x)\|_2^2\bigr]}.
\end{equation}
All errors use this distribution; choosing $p$ uniform on a dataset gives empirical guarantees on that dataset.
Further assumptions are introduced where needed.

We use RMSE because it aggregates discrepancies across inputs while retaining the scale of the quantities being compared.
It satisfies $\rmse(f_1,f_2)=\|f_1-f_2\|_{L^2(p)}$, where the $L^2(p)$ norm is defined on square-integrable functions identified up to equality $p$-almost surely.
We can therefore use properties of this norm, including the triangle inequality.
The Euclidean norms measure discrepancies between individual representations in $\mathcal{Z}$ or outputs in $\mathcal{Y}$, while the $L^2(p)$ norm turns these pointwise discrepancies into distances between functions.
More general norms could be used with corresponding regularity assumptions; we adopt this setting to keep the statements and proofs easy to follow.

\subsection{Fidelity}
\label{sec:cg-fidelity}

The \emph{fidelity error} measures output disagreement between $f$ and $f_A$.
Thus we define $\fiderror(A)=\rmse(f,f_A)$.
The \emph{reconstruction error} measures the disagreement between $z$ and $D(E(z))$ where $z = h(x)$, so we define $\recerror(A)=\rmse(h,D\circ C)$.

Now observe that if $h(x) = D(E(h(x)))$, then
\begin{equation}
    f(x) = g(h(x)) = g(D(E(h(x)))) = f_A(x).
\end{equation}
From this, we can easily see that if $\recerror(A) = 0$, then $\fiderror(A) = 0$ as well.
For approximate reconstruction, assume that $g$ is $L_g$-Lipschitz for a finite constant $L_g\geq0$:
\begin{equation}
\label{eq:cg-lipschitz}
    \|g(z)-g(z')\|_2\leq L_g\|z-z'\|_2
    \qquad\text{for all }z,z'\in\mathcal{Z}.
\end{equation}
This controls how strongly the prediction head can amplify reconstruction errors.

\begin{theorem}[Reconstruction controls fidelity]
\label{thm:cg-fidelity}
Under the Lipschitz assumption in \cref{eq:cg-lipschitz},
\begin{equation}
\label{eq:cg-fidelity-bound}
    \fiderror(A)\leq L_g\recerror(A).
\end{equation}
\end{theorem}

For a proof see \cref{app:cg-fidelity}.

This result reflects the construction of the ICAM: $f_A$ replaces $h(x)$ by its reconstruction $D(E(h(x)))$ before applying the same prediction head $g$.
The usefulness of this dwells in the formalization, as we show this is a stepping stone to more interesting results.
The Lipschitz assumption holds for many common neural-network heads, including finite compositions of affine maps and Lipschitz activations such as ReLU.
For deeper compositions, products of the individual layer bounds give a valid but potentially conservative constant.

\subsection{Model Completeness}
\label{sec:cg-model}

Model completeness asks how well the original outputs can be recovered from $C(x)$ when the prediction head is allowed to vary.
Let then $\Gamma$ contain all functions $\mathcal{V}\to\mathbb{R}^N$.
We define the \emph{model completeness error} as $\modelerror(A) =\inf_{\eta\in\Gamma}\rmse(f,\eta\circ C)$.

Now, observe that $g\circ D\in\Gamma$ is an admissible head.
The definition of the infimum therefore gives $\modelerror(A)\leq\rmse(f,g\circ D\circ C)=\fiderror(A)$.
Combining this with \cref{thm:cg-fidelity}, we get the following result.

\begin{theorem}[Reconstruction controls completeness]
\label{thm:cg-mce-fidelity-rec}
Under the Lipschitz assumption in \cref{eq:cg-lipschitz},
\begin{equation}
\label{eq:cg-mce-fidelity-rec}
    \modelerror(A) \leq \fiderror(A) \leq L_g\recerror(A).
\end{equation}
\end{theorem}

The connection is not immediately obvious, as completeness concerns only the concept representations $E(h(x))$, while reconstruction concerns both $E$ and $D$.
In particular, if $\recerror(A)=0$, then the concept representations are complete.
In this view, the reason is clear: $E$ retains enough information for $D$ to perfectly reconstruct the latent representation almost surely, and thus recover the original output.
This exact-reconstruction conclusion does not require the Lipschitz assumption.
The theorem also gives an additional interpretation to minimizing reconstruction error when learning $A$: under the Lipschitz assumption, it lowers an upper bound on the incompleteness of the concept representations.

\subsection{Attribution Completeness}
\label{sec:cg-attribution}

Take scalar outputs, $N=1$, and assume that each $\mathcal{V}_i\subseteq\mathbb{R}^{m_i}$ contains its ambient zero vector $0_i$, where $m_i\geq1$.
Equip $\mathcal{W}=\prod_i\mathbb{R}^{m_i}$ with its Euclidean inner product and norm.
Write $b=\gamma(0_{\mathcal{V}})$, where $0_{\mathcal{V}}=(0_1,\ldots,0_k)$.
The isolation map
    $P_i(v_1,\ldots,v_k)
    =(0_1,\ldots,0_{i-1},v_i,0_{i+1},\ldots,0_k)$
retains concept $i$; both $P_i v$ and the representation $v-P_i v$ with that concept removed belong to $\mathcal{V}$.

An \emph{attribution} is a function $\varphi=(\varphi_1,\ldots,\varphi_k)\colon\mathcal{X}\to\mathbb{R}^k$ that assigns each concept a contribution to a prediction.
We study three choices:
\begin{align}
\label{eq:cg-concept-attribution}
    \varphi_i(x)
    &=\gamma(P_iC(x))-b
    &&\text{(Insertion)},\\
\label{eq:cg-occlusion-attribution}
    \varphi_i(x)
    &=\gamma(C(x))-\gamma(C(x)-P_iC(x))
    &&\text{(Occlusion)},\\
\label{eq:cg-gradient-input-attribution}
    \varphi_i(x)
    &=\langle\nabla\gamma(C(x)),P_iC(x)\rangle
    &&\text{(Gradient-times-input)}.
\end{align}
Insertion adds a concept to the baseline, occlusion removes it from the full representation, and gradient-times-input pairs its coordinates with their output sensitivities.
For the last rule, $\gamma$ must also be defined and differentiable on an open neighborhood of each $C(x)$ in $\mathcal{W}$, agreeing with $g\circ D$ on its intersection with $\mathcal{V}$.

An attribution function $\varphi$ induces an \emph{attribution model} $f_\varphi\colon\mathcal{X}\to\mathbb{R}$, defined by
\begin{equation}
    f_\varphi(x)=b+\sum_{i=1}^k\varphi_i(x).
\end{equation}
The \emph{attribution error} measures output disagreement between $f$ and $f_\varphi$, so we define $\atterror(A,\varphi)=\rmse(f,f_\varphi)$.
Similarly, the \emph{additivity error} measures output disagreement between $f_A$ and $f_\varphi$, so we define $\adderror(A,\varphi)=\rmse(f_A,f_\varphi)$.

Observe that attributions computed through $\gamma=g\circ D$ concern the ICAM $f_A$, whose predictions may differ from those of the original model $f$.
Our attribution model $f_\varphi$ makes this distinction explicit: its agreement with $f_A$ is measured by additivity error, while its agreement with $f$ is measured by attribution error.
Applying the triangle inequality for RMSE to $f$, $f_A$, and $f_\varphi$ gives
\begin{equation}
\label{eq:cg-attribution-decomposition}
    \atterror(A,\varphi)\leq\fiderror(A)+\adderror(A,\varphi).
\end{equation}
If $\adderror(A,\varphi)=0$, then $f_\varphi=f_A$ almost surely, and hence $\atterror(A,\varphi)=\fiderror(A)$.
Thus, even attributions that perfectly recover the ICAM output retain its disagreement with the original model.
This lets us study additivity within the ICAM separately and use fidelity to connect the resulting guarantees to the original predictions.

In the following proposition, we provide a general case where $\adderror(A,\varphi)=0$.

\begin{proposition}[Complete attributions for affine decoded heads]
\label{prop:cg-affine-attribution}
Suppose that $\gamma$ is affine, i.e., there are $a_i\in\mathbb{R}^{m_i}$ such that $\gamma(v_1,\ldots,v_k)=b+\sum_{i=1}^k\langle a_i,v_i\rangle$ for all $v\in\mathcal{V}$.
Then the attribution $\varphi_i(x)=\langle a_i,\pi_i(h(x))\rangle$ has zero additivity error, i.e., $\adderror(A,\varphi)=0$.
\end{proposition}

For a proof see \cref{app:cg-affine-attribution}.
Insertion, occlusion, and gradient-times-input all give precisely this attribution for an affine decoded head.
We verify this in \cref{app:cg-affine-attribution-rules}.
Thus, all three rules have zero additivity error, and their attribution error equals fidelity error.

This is an important case because many concept discovery methods use affine decoders.
If $g$ is also affine, as when explaining a class logit through the final affine layer, then $\gamma=g\circ D$ is affine even when the encoder is nonlinear.
\cref{thm:cg-fidelity} therefore gives
\begin{equation}
    \atterror(A,\varphi)=\fiderror(A)\leq L_g\recerror(A)
\end{equation}
for all three rules.
For a scalar head $g(z)=\langle w,z\rangle+\beta$, we may take $L_g=\|w\|_2$, so this bound can be evaluated directly from the classifier weights and reconstruction error.
Small reconstruction error together with a small Lipschitz constant therefore guarantees that the attributions approximately recover the original model output, and exact reconstruction makes them complete.
This conclusion applies to the affine class scores; applying softmax generally makes the decoded head nonaffine.

\paragraph{Bounding the Nonaffinity of $\gamma$}
When explaining an earlier layer, the remaining prediction head $g$ may be nonlinear, and a nonlinear decoder can also make $\gamma=g\circ D$ nonaffine.
The three attribution rules may then give different contributions with nonzero additivity error.
To extend the affine case, we control how far $\gamma$ departs from an affine approximation by bounding the variation of its gradient.

Assume that $\gamma$ is also defined and differentiable on an open set $\Omega\subseteq\mathcal{W}$ containing the segments from zero to $C(x)$ and $P_iC(x)$, and from $C(x)$ to $C(x)-P_iC(x)$, for every $x$ and $i$.
Its values on $\Omega\cap\mathcal{V}$ agree with $g\circ D$.
Assume that $\mathbb{E}_{x\sim p}\bigl[\|C(x)\|_2^4\bigr]<\infty$ and that the gradient of $\gamma$ is $M$-Lipschitz for a finite $M\geq0$:
\begin{equation}
\label{eq:cg-gradient-lipschitz}
    \|\nabla\gamma(v)-\nabla\gamma(w)\|_2\leq M\|v-w\|_2
    \qquad\text{for all }v,w\in\Omega.
\end{equation}

\begin{theorem}[Curvature bounds additivity error]
\label{thm:cg-additivity-curvature}
Under these assumptions, insertion, occlusion, and gradient-times-input satisfy
\begin{equation}
\label{eq:cg-additivity-bound}
    \adderror(A,\varphi)
    \leq M\sqrt{\mathbb{E}_{x\sim p}\bigl[\|C(x)\|_2^4\bigr]}.
\end{equation}
\end{theorem}

A proof is in \cref{app:cg-additivity-curvature}.

\subsection{Empirical Tightness}
\label{sec:cg-empirical}

To assess how conservative these bounds are, we evaluate them on a ResNet-50 trained on ImageNet-1k, split after its last residual stage and before its final affine layer, where the prediction head is affine and $L_g$ is available in closed form.
We fit concept autoencoders with PCA, NMF, $k$-means, a sparse autoencoder, and an autoencoder with a nonlinear decoder, for $k\in\{5,25,50\}$ concepts and three seeds; details and all results are in \cref{sec:cg-tightness}.
All bounds hold in every configuration, but their tightness differs considerably.

The fidelity bound of \cref{thm:cg-fidelity} is conservative: the fidelity error is about a quarter of $L_g\recerror(A)$ before the final layer and about a tenth after the last residual stage, almost independently of the method and of $k$.
The slack has two sources.
First, reconstruction errors lie mostly in directions to which the prediction head is insensitive, whereas $L_g$ accounts for its most sensitive direction.
Second, after the last residual stage, the errors at different spatial locations partly cancel when the head averages over them.

For model completeness, heads fitted to the concept representations reduce the error of the decoded head $g\circ D$ by at most about $9\%$, so $g\circ D$ is already close to the best head we could find.

For attributions, the tightness depends on the decoder.
With the four affine decoders, insertion, occlusion, and gradient-times-input are additive up to floating-point error, so the attribution error equals the fidelity error and the attribution bound is attained, as \cref{prop:cg-affine-attribution} predicts.
With the nonlinear decoder, the attribution bound, which adds the curvature term of \cref{thm:cg-additivity-curvature} to the fidelity error, still holds but exceeds the attribution error by two to three orders of magnitude.

\section{Related Work}  
\label{sec:related-work}

\paragraph{Concepts from Dictionary Learning}
\label{sec:example-holistic}
\citet{fel2023holistic} unify $k$-means, PCA, NMF, and sparse autoencoders through dictionary learning, abstracting methods such as ACE~\citep{ghorbani_towards_2019}, ICE~\citep{zhang2021invertible}, and CRAFT~\citep{fel_craft_2023}.
The same encoder--decoder formulation also accommodates FACE~\citep{bhusal2025face}.
Their coefficient inference is our encoder $E$, their dictionary defines $D(v)=Bv$, and their reconstructed predictor is our ICAM.
Allowing a decoder bias also accommodates Anthropic's sparse autoencoders~\citep{bricken2023monosemanticity,templeton2024scaling}.
Our guarantees apply to these fitted constructions under the stated assumptions, independently of the learning objective or whether optimization reaches a global optimum.
For affine $g$, as in the holistic framework's penultimate-layer analysis, $\gamma=g\circ D$ is affine, so our results apply to insertion and to the occlusion and gradient-times-input rules they consider.
Our analysis complements their optimality guarantees for importance rankings and FACE's bounds on predictive disagreement by connecting reconstruction to model completeness and distinguishing attribution completeness from fidelity.
The formulation also permits nonlinear decoders, with approximate guarantees under the corresponding Lipschitz and curvature assumptions.

\paragraph{Multidimensional Concept Discovery}
\label{sec:related-mcd}
MCD~\citep{vielhaben2023multidimensional} represents concepts by subspaces of feature channels.
Its component projections are our probes, and its decoder sums the components: $D(v_1,\ldots,v_k)=\sum_i v_i$.
For the complete basis decomposition, including the residual, $D\circ E=\id_{\mathcal{Z}}$.
Together with its affine prediction head and blockwise relevance scores, this makes MCD's attribution-completeness relation a special case of exact reconstruction and \cref{prop:cg-affine-attribution}.
Our formulation identifies the structural conditions behind this result and additionally establishes model completeness of the full decomposition for any prediction head.
Our approximate bounds address discarded components and nonlinear heads under the corresponding regularity assumptions.
These distribution-dependent errors complement MCD's geometric completeness score based on classifier weights.
Our product representation spaces also accommodate HU-MCD~\citep{grobrugge2025human} and the vector-valued concept blocks of subspace-aware sparse autoencoders~\citep{dalili2026subspace}, without requiring the same discovery procedure or exact reconstruction.

\paragraph{Completeness and Evaluation}
\label{sec:related-completeness}
\citet{yeh2020completeness} measure completeness through normalized classification accuracy against ground-truth labels, optimizing a decoder for thresholded, normalized concept projections while retaining the original prediction head.
Our model completeness error instead measures recovery of the original outputs $f(x)$ over all measurable heads on $C(x)$, separating information retained by the encoder from the performance of a particular decoder.
Their ConceptSHAP distributes a global completeness score among concepts, whereas our attribution error concerns recovery of individual model outputs.
Other evaluations address semantic readability and perturbation-based faithfulness~\citep{li2024readability}, or whether explanations enable an LLM to simulate model predictions, as in ConSim~\citep{poche2025consim}.
SAEBench further shows that improvements in sparse-autoencoder proxy metrics need not improve practical interpretability evaluations~\citep{karvonen2025saebench}.
Our contribution is a common formulation and conditional guarantees for existing constructions; we do not introduce a discovery algorithm, and our guarantees do not establish semantic readability or successful simulation.

\section{Conclusion}
\label{sec:conclusion}

We introduced a common language for concept-based prediction, detection, and discovery, distinguishing input-level concepts, potentially multidimensional representations, and scalar activations.
For concept autoencoders, exact reconstruction gives zero fidelity and model completeness errors, while a Lipschitz prediction head provides quantitative guarantees for approximate reconstruction.
We also bound attribution error by fidelity error plus additivity error within the ICAM.
Insertion, occlusion, and gradient-times-input have zero additivity error for affine decoded heads, and our curvature bound extends the analysis to smooth nonaffine heads.

These guarantees apply to existing methods through their encoder--decoder structure, independently of their training objectives.
They concern output recovery and attribution completeness; the semantic meaning of concepts requires separate evaluation.

\subsection*{AI Use Statement}

We used generative AI tools to assist with developing and refining the conceptual framework, mathematical statements, and proofs.
We used a generative AI coding assistant to help implement, debug, and optimize the evaluation code.
We also used generative AI to support literature exploration, assist with analyzing experimental results, prepare figures, and draft and revise the structure, language, and formatting of the paper.

The authors reviewed the AI-assisted mathematical arguments, figures, code, and text.
The code was tested on small configurations, and its outputs were checked against theoretical predictions.
AI-assisted text was read and revised by the authors, and every reported number was checked against the logged results.
We take responsibility for the final content of this work, including text, claims, and artifacts produced with the aid of generative AI.

\subsection*{Reproducibility Statement}

All assumptions of our theoretical results are stated explicitly in \cref{sec:completeness-guarantees} and the proofs are either in the main text or in \cref{app:curvature-proofs}.
For the empirical evaluation, \cref{sec:cg-tightness} describes the model, data, boundaries, concept discovery methods, prediction heads, and the estimation of $M$, and reports all results as means over three seeds together with their standard deviations.
The evaluation uses only public resources: the torchvision ResNet-50 weights (\texttt{IMAGENET1K\_V2}) and the ImageNet-1k validation set.
We provide the source code as supplementary material; each run is fully specified by a configuration file and a seed, which determines the sampled images, the split into fitting and evaluation halves, and the initialization of all methods and heads.

\subsubsection*{Acknowledgments}

The work was supported by the Czech Science Foundation (GA\v{C}R) grant no.~26-23981S.
Computational resources were provided by the e-INFRA CZ project (ID:~90140), supported by the Ministry of Education, Youth, and Sports of the Czech Republic.

\appendix
\section{Instantiations of Concept-Based Methods}
\label{app:instantiations}

\cref{tab:instantiations} summarizes how the methods discussed in \cref{sec:background,sec:related-work} fit our framework.
For explainable-by-design models, the latent space is the representation space, $\mathcal{Z}=\mathcal{V}$, and the probes $\pi_i(z)=z_i$ form the identity encoder.
The trivial concept autoencoder $E=D=\id$ then gives $f_A=f$ and $\gamma=g$, so $\modelerror(A)=\fiderror(A)=0$, and attribution completeness depends only on whether $g$ is affine.
Concept detection methods learn probes for individual user-specified concepts without a decoder, so the guarantees of \cref{sec:completeness-guarantees} do not apply to them directly.
All concept discovery methods in the table use affine decoders, so $\gamma$ is affine whenever $g$ is.

\begin{table}[ht]
\caption{Instantiations of concept-based methods in our framework: CBMs~\citep{koh2020concept}, PCBM~\citep{yuksekgonul2023posthoc}, LF-CBM~\citep{oikarinen2023labelfree}, CEM~\citep{zarlenga2022concept}, TCAV~\citep{kim2018interpretability}, CAR~\citep{crabbe2022car}, ACE~\citep{ghorbani_towards_2019}, CRAFT~\citep{fel_craft_2023}, ICE~\citep{zhang2021invertible}, SAE~\citep{bricken2023monosemanticity}, and MCD~\citep{vielhaben2023multidimensional}.
Columns give the latent space $\mathcal{Z}$, the concept representation space $\mathcal{V}_i$, the probe $\pi_i$, the activation function $\sigma_i$, the decoder $D$, and whether the decoded head $\gamma=g\circ D$ is affine; ``if $g$'' means that $\gamma$ is affine whenever $g$ is.
A dash marks a component the method does not specify.
$s$ is the logistic sigmoid, $U$ is a dictionary with concept directions $u_i$ in its rows, NMF encoders compute nonnegative least-squares coefficients with $U$ fixed, and $Q_i$ projects channels onto the $i$-th MCD subspace, with the residual subspace included as a concept.}
\label{tab:instantiations}
\centering
\small
\setlength{\tabcolsep}{4pt}
\begin{tabular}{@{}lllllll@{}}
\toprule
Method & $\mathcal{Z}$ & $\mathcal{V}_i$ & $\pi_i(z)$ & $\sigma_i$ & $D$ & $\gamma$ affine \\
\midrule
\multicolumn{7}{@{}l}{\emph{Explainable by design}} \\
CBM, probabilities & $[0,1]^k$ & $[0,1]$ & $z_i$ & $\id$ & $\id$ & if $g$ \\
CBM, logits & $\mathbb{R}^k$ & $\mathbb{R}$ & $z_i$ & $s$ & $\id$ & if $g$ \\
PCBM, LF-CBM & $\mathbb{R}^k$ & $\mathbb{R}$ & $z_i$ & -- & $\id$ & yes \\
CEM & $\prod_i\mathbb{R}^{2m}$ & $\mathbb{R}^{2m}$ & $z_i$ & $s(\langle w,v\rangle+\beta)$ & $\id$ & no \\
\midrule
\multicolumn{7}{@{}l}{\emph{Concept detection}} \\
TCAV & $\mathbb{R}^d$ & $\mathbb{R}$ & $\langle w,z\rangle+\beta$ & $\mathbf{1}\{t>0\}$ & -- & -- \\
CAR & $\mathbb{R}^d$ & $\mathbb{R}$ & kernel SVM & $\mathbf{1}\{t>0\}$ & -- & -- \\
\midrule
\multicolumn{7}{@{}l}{\emph{Concept discovery}} \\
ACE & $\mathbb{R}^d$ & $\{0,1\}$ & $\mathbf{1}\{i=\arg\min_j\|z-u_j\|\}$ & -- & $U^\top v$ & if $g$ \\
CRAFT & $\mathbb{R}^d_{\geq0}$ & $\mathbb{R}_{\geq0}$ & $E(z)_i$ (NMF) & -- & $U^\top v$ & if $g$ \\
ICE & $\mathbb{R}^{S\times d}_{\geq0}$ & $\mathbb{R}^{S}_{\geq0}$ & $E(z)_{i,:}$ (NMF) & -- & $V^\top U$ & yes \\
SAE & $\mathbb{R}^d$ & $\mathbb{R}_{\geq0}$ & $\mathrm{ReLU}(\langle w_i,z\rangle+\beta_i)$ & -- & $U^\top v+\beta_D$ & if $g$ \\
MCD & $\mathbb{R}^{S\times d}$ & $\mathbb{R}^{S\times d}$ & $zQ_i^\top$ & -- & $\sum_i v_i$ & if $g$ \\
\bottomrule
\end{tabular}
\end{table}

For CEM, we take as the representation of concept $i$ the pair of positive and negative embeddings $(\hat c_i^+,\hat c_i^-)\in\mathbb{R}^{2m}$, the shared scoring function as $\sigma_i$, and include the probability-weighted mixing of the two embeddings in $g$; this mixing makes $g$ nonaffine.
Following \citet{fel2023holistic}, we read ACE's clustering of segment activations as $k$-means dictionary learning with one-hot coefficients.

\paragraph{ICE}
\cref{fig:ice-diagram} summarizes the ICE example of \cref{sec:background-unsupervised}.
Since $D$ and $g$ are both affine, so is the decoded head,
\begin{equation}
\label{eq:app-ice-head}
    \gamma(V)
    =g(V^\top U)
    =b+\frac{1}{S}\sum_{s=1}^{S}\sum_{i=1}^{k}V_{i,s}\langle u_i,w\rangle
    =b+\sum_{i=1}^{k}(Uw)_i\,\bar v_i,
    \qquad
    \bar v_i=\frac{1}{S}\sum_{s=1}^{S}V_{i,s},
\end{equation}
where $u_i$ is the $i$-th row of $U$.
This is the affine form of \cref{prop:cg-affine-attribution} with $a_i=S^{-1}(Uw)_i\mathbf{1}_S$.
Insertion, occlusion, and gradient-times-input therefore all return ICE's contributions $\varphi_i(x)=(Uw)_i\,\bar v_i$, which sum with $b$ to $f_A(x)$, so $\atterror(A,\varphi)=\fiderror(A)$.
The head $g$ is Lipschitz with $L_g=\|w\|_2/\sqrt{S}$, and \cref{thm:cg-fidelity} gives $\fiderror(A)\leq\|w\|_2\recerror(A)/\sqrt{S}$.

\begin{figure}[ht]
    \centering
    \begin{tikzcd}[column sep=large, row sep=large]
        \mathcal{X} \arrow[r, "h"] & \mathbb{R}^{S\times d}_{\geq0} \arrow[r, "g"] \arrow[d, shift right, "E"'] & \mathbb{R} \\[-0.5em]
        \mathbb{R}^{S}_{\geq0} & \mathbb{R}^{k\times S}_{\geq0} \arrow[u, shift right, "D"'] \arrow[l, "\pi_i"'] \arrow[r, "\Phi"'] & \mathbb{R}^{k} \arrow[u, "b+\sum_i"']
    \end{tikzcd}
    \caption{ICE~\citep{zhang2021invertible} as a concept autoencoder.
    The encoder $E$ computes nonnegative coefficients $V=E(z)$ for the fixed dictionary $U\in\mathbb{R}^{k\times d}_{\geq0}$, and the decoder is $D(V)=V^\top U$.
    The probe $\pi_i$ returns the spatial coefficient map of concept $i$, which ICE uses to highlight image regions.
    The prediction head is the class logit $g(z)=b+\frac{1}{S}\sum_s\langle w,z_{s,:}\rangle$, and $\Phi(V)=\bigl((Uw)_i\,\bar v_i\bigr)_{i=1}^k$ returns the concept contributions.
    The right square commutes, $g\circ D=b+\sum_i\Phi_i$, so the contributions sum to $f_A(x)$ exactly; since $D\circ E$ only approximates the identity, they recover $f(x)$ up to $\fiderror(A)$.}
    \label{fig:ice-diagram}
\end{figure}

\section{Proofs for \cref{sec:completeness-guarantees}}
\label{app:curvature-proofs}

We use the notation and assumptions of \cref{sec:completeness-guarantees}.

\subsection{Proof of \cref{thm:cg-fidelity}}
\label{app:cg-fidelity}

\begin{proof}
For every $x\in\mathcal{X}$, Lipschitz continuity gives
\begin{equation}
\label{eq:cg-fidelity-pointwise-bound}
    \|f(x)-f_A(x)\|_2
    =\|g(h(x))-g(D(C(x)))\|_2
    \leq L_g\|h(x)-D(C(x))\|_2.
\end{equation}
Squaring and taking expectations preserves the inequality by monotonicity of expectation.
Linearity of expectation allows $L_g^2$ to be taken outside, and taking square roots gives the result.
\end{proof}

\subsection{Proof of \cref{prop:cg-affine-attribution}}
\label{app:cg-affine-attribution}

\begin{proof}
For every $x\in\mathcal{X}$, the definitions of $f_\varphi$ and $C$ give
\begin{equation*}
    f_\varphi(x)
    =b+\sum_{i=1}^k\varphi_i(x)
    =b+\sum_{i=1}^k\langle a_i,\pi_i(h(x))\rangle
    =\gamma(C(x))
    =f_A(x).
\end{equation*}
Thus $f_\varphi=f_A$, and consequently $\adderror(A,\varphi)=\rmse(f_A,f_\varphi)=0$.
\end{proof}

\subsection{Attribution Rules for Affine Decoded Heads}
\label{app:cg-affine-attribution-rules}

Suppose that $\gamma$ has the affine form in \cref{prop:cg-affine-attribution}.
For insertion, only block $i$ is retained, so
\[
    \gamma(P_iC(x))-b
    =\langle a_i,\pi_i(h(x))\rangle.
\]
For occlusion, subtracting the prediction with block $i$ removed cancels the baseline and all other block contributions:
\[
    \gamma(C(x))-\gamma(C(x)-P_iC(x))
    =\langle a_i,\pi_i(h(x))\rangle.
\]
For gradient-times-input, we use the affine formula on $\mathcal{W}$, whose gradient is the constant vector $(a_1,\ldots,a_k)$.
Consequently,
\[
    \langle\nabla\gamma(C(x)),P_iC(x)\rangle
    =\langle a_i,\pi_i(h(x))\rangle.
\]
All three rules therefore satisfy the attribution formula of \cref{prop:cg-affine-attribution}.

\subsection{Proof of \cref{thm:cg-additivity-curvature}}
\label{app:cg-additivity-curvature}

\begin{proof}
Choose $a(x)=\nabla\gamma(0_{\mathcal{V}})$ for insertion and $a(x)=\nabla\gamma(C(x))$ for occlusion and gradient-times-input.
For each input $x$, define the affine approximation
\begin{equation}
\label{eq:cg-affine-approximation}
    \gamma_{\mathrm{aff},x}(v)=b+\langle a(x),v\rangle,
    \qquad v\in\mathcal{W}.
\end{equation}
The bias is fixed at $b$, so every approximation agrees with $\gamma$ at zero.
Define $f_{\mathrm{aff}}(x)=\gamma_{\mathrm{aff},x}(C(x))$.
The triangle inequality for RMSE gives
\begin{equation}
\label{eq:cg-additivity-split}
    \adderror(A,\varphi)
    \leq\rmse(f_A,f_{\mathrm{aff}})
    +\rmse(f_{\mathrm{aff}},f_\varphi).
\end{equation}

The standard Taylor bound for a function with an $M$-Lipschitz gradient gives
\begin{equation}
\label{eq:cg-remainder-bound}
    |\gamma(v)-\gamma(w)-\langle\nabla\gamma(w),v-w\rangle|
    \leq\frac{M}{2}\|v-w\|_2^2
\end{equation}
whenever the segment from $w$ to $v$ lies in $\Omega$.
We use this bound to control both terms in \cref{eq:cg-additivity-split}.

For the first term, apply \cref{eq:cg-remainder-bound} from zero to $C(x)$ for insertion, and from $C(x)$ to zero for the other two rules.
Using $b=\gamma(0_{\mathcal{V}})$ and the chosen $a(x)$, both cases give
\begin{equation}
\label{eq:cg-full-affine-pointwise}
    |f_A(x)-f_{\mathrm{aff}}(x)|
    =|\gamma(C(x))-b-\langle a(x),C(x)\rangle|
    \leq\frac{M}{2}\|C(x)\|_2^2.
\end{equation}
Squaring and taking expectations preserves the inequality by monotonicity of expectation.
By linearity of expectation, the constant $M^2/4$ can be taken outside, and taking square roots yields
\begin{equation}
\label{eq:cg-full-affine-error}
    \rmse(f_A,f_{\mathrm{aff}})
    \leq\frac{M}{2}
    \sqrt{\mathbb{E}_{x\sim p}\bigl[\|C(x)\|_2^4\bigr]}.
\end{equation}

For the second term, fix $x\in\mathcal{X}$.
The isolated blocks satisfy
\begin{equation}
\label{eq:cg-block-norms}
    C(x)=\sum_{i=1}^k P_iC(x),
    \qquad
    \|C(x)\|_2^2=\sum_{i=1}^k\|P_iC(x)\|_2^2.
\end{equation}
By linearity of the inner product, the affine prediction therefore decomposes as
\begin{equation}
\label{eq:cg-affine-additivity}
    f_{\mathrm{aff}}(x)
    =b+\sum_{i=1}^k\langle a(x),P_iC(x)\rangle.
\end{equation}
We compare each attribution $\varphi_i(x)$ with its corresponding affine contribution $\langle a(x),P_iC(x)\rangle$.

For insertion, applying \cref{eq:cg-remainder-bound} from zero to $P_iC(x)$ gives
\begin{equation*}
    |\gamma(P_iC(x))-b-\langle\nabla\gamma(0_{\mathcal{V}}),P_iC(x)\rangle|
    \leq\frac{M}{2}\|P_iC(x)\|_2^2.
\end{equation*}
For occlusion, applying the same bound from $C(x)$ to $C(x)-P_iC(x)$ and rearranging gives
\begin{equation*}
\begin{aligned}
    &|\gamma(C(x))-\gamma(C(x)-P_iC(x))
      -\langle\nabla\gamma(C(x)),P_iC(x)\rangle|\\
    &\qquad\leq\frac{M}{2}\|P_iC(x)\|_2^2.
\end{aligned}
\end{equation*}
For gradient-times-input, the attribution equals $\langle a(x),P_iC(x)\rangle$ by definition.
Thus, for all three rules,
\begin{equation}
\label{eq:cg-concept-affine-error}
    |\varphi_i(x)-\langle a(x),P_iC(x)\rangle|
    \leq\frac{M}{2}\|P_iC(x)\|_2^2.
\end{equation}
Using the common baseline $b$ and summing over concepts gives
\begin{align*}
    |f_{\mathrm{aff}}(x)-f_\varphi(x)|
    &=\left|\sum_{i=1}^k
      \bigl(\langle a(x),P_iC(x)\rangle-\varphi_i(x)\bigr)\right|\\
    &\leq\sum_{i=1}^k
      |\langle a(x),P_iC(x)\rangle-\varphi_i(x)|\\
    &\leq\frac{M}{2}\sum_{i=1}^k\|P_iC(x)\|_2^2
    =\frac{M}{2}\|C(x)\|_2^2.
\end{align*}
The inequalities use the triangle inequality and \cref{eq:cg-concept-affine-error}, and the last equality uses \cref{eq:cg-block-norms}.
Taking the root mean square as above gives
\begin{equation}
\label{eq:cg-isolated-affine-error}
    \rmse(f_{\mathrm{aff}},f_\varphi)
    \leq\frac{M}{2}
    \sqrt{\mathbb{E}_{x\sim p}\bigl[\|C(x)\|_2^4\bigr]}.
\end{equation}
Substituting \cref{eq:cg-full-affine-error,eq:cg-isolated-affine-error} into \cref{eq:cg-additivity-split} proves
\[
    \adderror(A,\varphi)
    \leq M\sqrt{\mathbb{E}_{x\sim p}\bigl[\|C(x)\|_2^4\bigr]}.
    \qedhere
\]
\end{proof}

\section{Evaluation of Tightness of the Bounds}
\label{sec:cg-tightness}

We evaluate how closely the bounds of \cref{sec:completeness-guarantees} are attained by concept autoencoders obtained with common concept discovery methods.

\paragraph{Setup}
We use a ResNet-50 trained on ImageNet-1k and, for each of three random seeds, $25{,}000$ images drawn from the ImageNet-1k validation set, which define the empirical distribution $p$.
Half of the images are used to fit the concept autoencoders and prediction heads, and all reported errors are computed on the other half.
We split the model at two boundaries: after the last residual stage (\emph{layer4}, $\mathcal{Z}=\mathbb{R}^{S\times d}$ with $S=49$ and $d=2048$) and before the final affine layer (\emph{penultimate}, $\mathcal{Z}=\mathbb{R}^{d}$).
At both boundaries the prediction head is affine, $g(z)=W\bar z+\beta$, where $\bar z=S^{-1}\sum_{s=1}^S z_s$ is the spatial average at layer4 and $\bar z=z$ at the penultimate boundary.
Consequently, the Lipschitz constant of \cref{thm:cg-fidelity} is available in closed form: $L_g=\|W\|_2$ (the largest singular value of $W$) at the penultimate boundary, and $L_g=\|W\|_2/\sqrt{S}$ at layer4, since averaging over $S$ locations shrinks the Euclidean norm on $\mathbb{R}^{S\times d}$ by at least a factor $\sqrt{S}$, with equality for errors that are identical at every location.
We extract $k\in\{5,25,50\}$ concepts with five methods that cover different decoder types: PCA (orthogonal linear decoder), NMF (nonnegative linear decoder), $k$-means (centroid decoder with soft assignments as concept representations), a sparse autoencoder (affine decoder), and an autoencoder with a nonlinear decoder.
Methods acting on spatial latents are applied to each spatial feature vector separately, so each concept has one coefficient per location.
PCA, NMF, and $k$-means use their scikit-learn implementations.
Both autoencoders share the encoder $\mathbb{R}^d\to\mathbb{R}^{4k}\to\mathbb{R}^k$ with a ReLU after each layer; the SAE has the affine decoder $\mathbb{R}^k\to\mathbb{R}^d$ and an $\ell_1$ penalty on the concept representation, and the nonlinear autoencoder has the decoder $\mathbb{R}^k\to\mathbb{R}^{4k}\to\mathbb{R}^d$ with a ReLU after the hidden layer.

\paragraph{Tightness Ratios}
Since the errors differ by orders of magnitude across boundaries and methods, we report each bound through the ratio of its left-hand side to its right-hand side,
\begin{equation}
\label{eq:cg-tightness-ratios}
    \rho_{\fiderror}=\frac{\fiderror(A)}{L_g\recerror(A)},
    \qquad
    \rho_{\atterror}=\frac{\atterror(A,\varphi)}{\fiderror(A)+M\sqrt{\mathbb{E}_{x\sim p}\bigl[\|C(x)\|_2^4\bigr]}}.
\end{equation}
A ratio of one means that the bound is attained, and smaller ratios indicate slack.
The corresponding ratio $\modelerror(A)/\fiderror(A)$ for model completeness cannot be computed, since $\modelerror$ is an infimum over all heads, and we report the upper estimate $\hat\rho_{\modelerror}$ described below.
\cref{tab:cg-tightness} reports the fidelity and model completeness ratios for every boundary, method, and $k$, and \cref{tab:cg-attribution} reports the attribution ratios for the nonlinear autoencoder, the only method for which they differ from one.

\begin{table}[t]
\caption{Tightness of the fidelity and model completeness bounds, mean over three seeds; the standard deviation over seeds is at most $0.012$ for $\hat\rho_{\modelerror}$ and at most $0.002$ for the other ratios.
$\rho_{\fiderror}=\alpha\kappa$ is the tightness of \cref{thm:cg-fidelity}, with alignment $\alpha$ and spatial coherence $\kappa$ from \cref{eq:cg-fidelity-slack}.
$\hat\rho_{\modelerror}$ upper-bounds the tightness of \cref{thm:cg-mce-fidelity-rec}.}
\label{tab:cg-tightness}
\centering
\setlength{\tabcolsep}{12pt}
\begin{tabular}{llrcccc}
\toprule
Boundary & Method & $k$ & $\rho_{\fiderror}$ & $\alpha$ & $\kappa$ & $\hat\rho_{\modelerror}$ \\
\midrule
 &  & 5 & 0.106 & 0.270 & 0.394 & 0.967 \\
 & PCA & 25 & 0.103 & 0.266 & 0.386 & 0.927 \\
 &  & 50 & 0.099 & 0.261 & 0.381 & 0.939 \\
\cmidrule(lr){2-7}
 &  & 5 & 0.106 & 0.269 & 0.395 & 0.975 \\
 & NMF & 25 & 0.102 & 0.264 & 0.389 & 0.960 \\
 &  & 50 & 0.101 & 0.263 & 0.385 & 0.964 \\
\cmidrule(lr){2-7}
 &  & 5 & 0.105 & 0.267 & 0.394 & 0.993 \\
layer4 & $k$-means & 25 & 0.103 & 0.266 & 0.388 & 0.983 \\
 &  & 50 & 0.102 & 0.264 & 0.385 & 0.978 \\
\cmidrule(lr){2-7}
 &  & 5 & 0.106 & 0.270 & 0.392 & 0.968 \\
 & SAE & 25 & 0.102 & 0.265 & 0.386 & 0.934 \\
 &  & 50 & 0.099 & 0.261 & 0.381 & 0.943 \\
\cmidrule(lr){2-7}
 &  & 5 & 0.104 & 0.267 & 0.390 & 0.985 \\
 & Nonlinear AE & 25 & 0.096 & 0.256 & 0.375 & 0.979 \\
 &  & 50 & 0.092 & 0.250 & 0.367 & 0.971 \\
\cmidrule(lr){1-7}
 &  & 5 & 0.270 & 0.270 & $1$ & 0.956 \\
 & PCA & 25 & 0.265 & 0.265 & $1$ & 0.908 \\
 &  & 50 & 0.261 & 0.261 & $1$ & 0.916 \\
\cmidrule(lr){2-7}
 &  & 5 & 0.270 & 0.270 & $1$ & 0.969 \\
 & NMF & 25 & 0.264 & 0.264 & $1$ & 0.944 \\
 &  & 50 & 0.261 & 0.261 & $1$ & 0.949 \\
\cmidrule(lr){2-7}
 &  & 5 & 0.267 & 0.267 & $1$ & 1.000 \\
penultimate & $k$-means & 25 & 0.264 & 0.264 & $1$ & 1.000 \\
 &  & 50 & 0.262 & 0.262 & $1$ & 1.000 \\
\cmidrule(lr){2-7}
 &  & 5 & 0.269 & 0.269 & $1$ & 0.978 \\
 & SAE & 25 & 0.265 & 0.265 & $1$ & 0.916 \\
 &  & 50 & 0.264 & 0.264 & $1$ & 0.916 \\
\cmidrule(lr){2-7}
 &  & 5 & 0.268 & 0.268 & $1$ & 0.975 \\
 & Nonlinear AE & 25 & 0.257 & 0.257 & $1$ & 0.969 \\
 &  & 50 & 0.250 & 0.250 & $1$ & 0.980 \\
\bottomrule
\end{tabular}
\end{table}

\paragraph{Fidelity}
The fidelity bound holds in every configuration, with $\rho_{\fiderror}$ between $0.092$ and $0.106$ at layer4 and between $0.250$ and $0.270$ at the penultimate boundary.
The slack has a structural explanation.
Write $\varepsilon(x)=h(x)-D(C(x))$ for the reconstruction error and $\bar\varepsilon(x)$ for its spatial average, so that $f(x)-f_A(x)=W\bar\varepsilon(x)$.
The bound then follows from two inequalities,
\begin{equation}
\label{eq:cg-fidelity-chain}
    \fiderror(A)
    =\sqrt{\mathbb{E}\|W\bar\varepsilon\|_2^2}
    \leq\|W\|_2\sqrt{\mathbb{E}\|\bar\varepsilon\|_2^2}
    \leq\frac{\|W\|_2}{\sqrt{S}}\sqrt{\mathbb{E}\|\varepsilon\|_2^2}
    =L_g\recerror(A).
\end{equation}
The first holds because $W$ stretches no vector by more than its largest singular value $\|W\|_2$, and the second because averaging over $S$ locations shrinks the Euclidean norm by at least a factor $\sqrt{S}$.
The tightness ratios of the two steps, the \emph{alignment} $\alpha$ and the \emph{spatial coherence} $\kappa$, multiply to the tightness ratio of the bound,
\begin{equation}
\label{eq:cg-fidelity-slack}
    \rho_{\fiderror}=\alpha\kappa,
    \qquad
    \alpha=\frac{\sqrt{\mathbb{E}\|W\bar\varepsilon\|_2^2}}{\|W\|_2\sqrt{\mathbb{E}\|\bar\varepsilon\|_2^2}},
    \qquad
    \kappa=\frac{\sqrt{S\,\mathbb{E}\|\bar\varepsilon\|_2^2}}{\sqrt{\mathbb{E}\|\varepsilon\|_2^2}},
\end{equation}
and both lie in $[0,1]$ by \cref{eq:cg-fidelity-chain}.
The alignment $\alpha$ measures how much of the averaged error lies along the directions that $W$ amplifies most, its top right singular vectors, and the spatial coherence $\kappa$ measures how much of the error survives averaging over locations, with $\kappa=1$ at the penultimate boundary.
Both factors equal one for an error that is identical at every location and points along the direction that $W$ amplifies most, so the constant $L_g$ in \cref{thm:cg-fidelity} cannot be improved without assumptions on the error.
Empirically, the alignment is nearly the same at both boundaries and for all methods, $\alpha\approx0.26$, so the reconstruction error is spread over many directions rather than concentrated along the directions that $W$ amplifies most.
At layer4, averaging over locations removes a further part of the error, $\kappa\approx0.38$, which explains the additional slack at this boundary.
As $k$ grows, $\recerror$ and $\fiderror$ decrease for all methods, and the ratio $\rho_{\fiderror}$ decreases slightly, on average from $0.105$ to $0.099$ at layer4 and from $0.269$ to $0.260$ at the penultimate boundary between $k=5$ and $k=50$.

\paragraph{Model Completeness}
Model completeness error is an infimum over all heads and cannot be computed exactly.
Every fitted head $\hat\gamma$ is admissible, so its held-out error $\widehat{\modelerror}=\rmse(f,\hat\gamma\circ C)$ estimates an upper bound on $\modelerror(A)$.
We fit two heads on the fitting half.
The first refits the decoded head within its own family, starting from $g\circ D$: when $D$ is affine on each location, the functions $g\circ D$ are exactly the affine maps of the averaged representation $\bar v$, so we fit $\hat\gamma(v)=\hat W\bar v+\hat\beta$ by least squares, and for the nonlinear decoder we fine-tune copies of $D$ and $g$ by gradient descent.
The second adds a learned correction to the decoded head, $\hat\gamma=g\circ D+r$, where $r$ is a multilayer perceptron initialized to zero.
Both heads start from $g\circ D$, and we select them by validation error, so they do not exceed $\fiderror(A)$ beyond sampling error.
We report the smaller of the two as $\widehat{\modelerror}$ and $\hat\rho_{\modelerror}=\widehat{\modelerror}/\fiderror(A)$, which upper-bounds the true tightness of \cref{thm:cg-mce-fidelity-rec}.
Consequently, our results can certify slack in this bound, since $\fiderror(A)-\widehat{\modelerror}$ lower-bounds $\fiderror(A)-\modelerror(A)$, but they cannot certify its tightness.
The gap is modest throughout: $\hat\rho_{\modelerror}$ ranges from $0.91$ to $1.00$.
It is largest for PCA and the SAE, and a nonlinear correction reduces the error by at most $9\%$ over the refitted decoded head.
For $k$-means at the penultimate boundary, $\hat\rho_{\modelerror}=1$: its assignments are nearly one-hot, and $g\circ D$ maps each cluster to the mean of $f$ over that cluster, which is already the best any head can do.

\begin{table}[t]
\caption{Tightness of the attribution bound for the nonlinear autoencoder, mean over three seeds.
Bound is $\fiderror(A)+M\sqrt{\mathbb{E}\bigl[\|C(x)\|_2^4\bigr]}$ with $\fiderror$ on the explained scores, shared by insertion (Ins.), occlusion (Occ.), and gradient-times-input (G$\times$I), and $\rho_{\atterror}$ is the tightness of the attribution bound obtained by combining \cref{eq:cg-attribution-decomposition} with \cref{thm:cg-additivity-curvature}.
The bound, and with it $\rho_{\atterror}$, varies across seeds by up to $40\%$ because $M$ is estimated by sampling, whereas $\adderror$ varies by at most $0.4$.
For the four methods with affine decoders, $\adderror<2\cdot10^{-5}$ and $\rho_{\atterror}=1$ for all three attribution functions.}
\label{tab:cg-attribution}
\centering
\setlength{\tabcolsep}{10pt}
\begin{tabular}{lrrcccccc}
\toprule
 & & & \multicolumn{3}{c}{$\adderror(A,\varphi)$} & \multicolumn{3}{c}{$\rho_{\atterror}\times10^{3}$} \\
\cmidrule(lr){4-6}\cmidrule(lr){7-9}
Boundary & $k$ & Bound & Ins. & Occ. & G$\times$I & Ins. & Occ. & G$\times$I \\
\midrule
 & 5 & 2281 & 1.31 & 1.39 & 0.10 & 3.0 & 3.0 & 2.9 \\
layer4 & 25 & 2381 & 2.34 & 2.92 & 0.06 & 2.9 & 1.9 & 2.4 \\
 & 50 & 2749 & 3.49 & 4.20 & 0.06 & 2.7 & 1.4 & 2.0 \\
\cmidrule(lr){1-9}
 & 5 & 2593 & 0.78 & 0.87 & 0.25 & 3.0 & 2.8 & 2.9 \\
penultimate & 25 & 1909 & 1.70 & 3.32 & 0.13 & 3.5 & 2.2 & 2.9 \\
 & 50 & 1726 & 2.40 & 4.46 & 0.09 & 3.9 & 2.2 & 2.9 \\
\bottomrule
\end{tabular}
\end{table}

\paragraph{Attribution}
We explain the score of the predicted class for each input and evaluate all three attribution functions of \crefrange{eq:cg-concept-attribution}{eq:cg-gradient-input-attribution}.
The theory assumes a scalar output, whereas the explained class varies between inputs.
The bounds still apply, because the triangle inequality in \cref{eq:cg-attribution-decomposition} and the Taylor bound in the proof of \cref{thm:cg-additivity-curvature} hold for each input and every class separately; we therefore compute $\fiderror$ on the explained scores only and take $M$ as the largest curvature constant over the explained classes.
For the four methods with affine decoders, $\gamma=g\circ D$ is a composition of affine maps, since $g$ is affine at both boundaries, and is therefore affine even though the SAE and $k$-means have nonlinear encoders.
In agreement with \cref{prop:cg-affine-attribution}, the additivity error vanishes up to floating-point error for all three attribution functions, $\adderror(A,\varphi)<2\cdot10^{-5}$ compared with $\fiderror\approx6$ on the explained scores, and the attribution error coincides with the fidelity error, so $\rho_{\atterror}=1$.
Only the nonlinear decoder produces curvature.
We estimate $M$ from the largest ratio of gradient differences over sampled pairs of nearby concept representations, which lower-bounds the true constant, so the reported bound is itself an estimate.
The curvature term does not depend on the scale of the concept representations: replacing $E$ by $\lambda E$ and $D$ by $D(\cdot/\lambda)$ for $\lambda>0$ leaves $f_A$ and the attributions unchanged, scales $M$ by $\lambda^{-2}$, and scales $\sqrt{\mathbb{E}\bigl[\|C(x)\|_2^4\bigr]}$ by $\lambda^2$, so the looseness of the bound reflects the curvature of $\gamma$ rather than the units of the concept representations.
For the nonlinear autoencoder (\cref{tab:cg-attribution}), both bounds hold for all three attribution functions but are loose, with $\rho_{\atterror}$ between $0.001$ and $0.004$.
Gradient-times-input is nearly additive, with $\adderror\leq0.25$, whereas insertion and occlusion reach $3.5$ and $4.5$ at $k=50$.
Better fidelity does not imply better attributions: as $k$ grows, $\fiderror$ on the explained scores decreases while $\atterror$ for insertion stays nearly constant, and for occlusion $\atterror$ even falls below $\fiderror$ because the additivity error partly cancels the fidelity error.

\paragraph{Limitations of the Evaluation}
All errors are estimated on a finite held-out sample.
Model completeness error is only bounded from above, and $M$ is estimated from below.
The closed-form Lipschitz constant restricts the evaluation to boundaries with affine prediction heads, and at earlier boundaries $L_g$ would itself have to be bounded.

\end{document}